\PassOptionsToPackage{sort}{natbib}
\documentclass{article}
\usepackage[dblblindworkshop,final]{neurips_2026}
\setcitestyle{square,semicolon}  % keep the square brackets, separate entries with ;

\usepackage[utf8]{inputenc}
\usepackage[T1]{fontenc}
\usepackage{amsmath, amssymb, amsthm}
\usepackage{amsfonts}
\usepackage{bm}
\usepackage{graphicx}
\usepackage{booktabs}
\usepackage{capt-of}
\usepackage{nicefrac}
\usepackage{microtype}
\usepackage{url}
\usepackage{array}
\usepackage{ragged2e}
\usepackage{tabularx}
\usepackage{xcolor}
\usepackage[framemethod=default]{mdframed}
\usepackage{hyperref}
\hypersetup{
  hidelinks,
  pdftitle={Binarization Flattens the Score Space},
  pdfauthor={Jacob Cole}
}

\newcolumntype{Y}{>{\RaggedRight\arraybackslash}X}

\theoremstyle{plain}
\newtheorem{theorem}{Theorem}
\newtheorem{corollary}{Corollary}
\newtheorem{proposition}{Proposition}
\newtheorem{lemma}{Lemma}
\theoremstyle{definition}
\newtheorem{definition}{Definition}
\newtheorem{observation}{Observation}

\newcounter{contrib}
\newcommand{\ind}{\mathbf{1}}
\newcommand{\Refset}{\Theta_{\mathrm{ref}}}

\newmdenv[
  linewidth=0.6pt,
  linecolor=black!55,
  backgroundcolor=black!3,
  innertopmargin=5pt,
  innerbottommargin=5pt,
  innerleftmargin=7pt,
  innerrightmargin=7pt,
  skipabove=6pt,
  skipbelow=4pt
]{claimbox}

\title{Binarization Flattens the Score Space}

\workshoptitle{JUDGe: Can We Trust the Judge?}
\author{Jacob Cole \\
  \texttt{jakecole@berkeley.edu}}

\begin{document}
\maketitle

\begin{abstract}
Large language model (LLM) judges are often used as rewards to train policies on objectives that
deterministic verifiers cannot capture. However, these rewards are often collapsed to pass/fail (\{0, 1\}), which reports the verdict but not how well a response met each criterion. We model each pass/fail verdict as a score on an unreported scale, compared with one cutoff. A \emph{stretch} of that scale moves every score proportionally toward or
away from the cutoff, but never across it, so no verdict changes. Such a stretch
is therefore invisible to everything the panel reports. Under a joint-Gaussian model, a third grade adds a second
threshold and removes this affine stretch ambiguity. On MATH and SciBench
outputs from one seven-criterion judge, all 14 constructed criterionwise
stretches were invisible after binarization but visible with three grades. At \(n=1{,}024\), a test given both population laws had at least \(96.5\%\) power at a \(1.5\times\) stress and a \(5\%\) test level. Retaining grades closes one blind spot created by binarization, but verdicts alone remain insufficient as some changes are still indistinguishable from genuine improvement. These include arbitrary within-grade changes and fixed-covariance, loading-aligned mean shifts---the signature of a sycophancy-shaped lift the panel reads as competence. The reference model fit MATH and SciBench approximately but failed on HealthBench. We recommend keeping at least three grades (for example, asking the judge whether each criterion is fully, partially, or not met and rewarding \{0, 0.5, 1\}), and externally validating gains along the remaining direction, which no finer scale removes.
\end{abstract}

% =====================================================================
\section{Introduction}
\label{sec:intro}

LLM-based evaluators are increasingly used as post-training rewards because
they can represent objectives and user preferences that are difficult to
express with programmatic verifiers. When policies optimize against these
rewards during post-training, judge errors can
propagate into model weights and downstream behavior. When the judge's output channel
supplies the reward, optimization can select for policy changes along directions
the channel cannot distinguish. Therefore, validity is not only a property of the judge, but also of its deployment system. Binary rubric rewards are
already used in judge-guided RL. For example, Compute as Teacher scores each criterion
yes/no and rewards the fraction satisfied \citep{jayalath2025compute}. Prior
work shows that
these rewards can rise while intended behavior
stagnates or degrades, including in controlled demonstrations of reward hacking
\citep{amodei2016concrete,skalse2022defining,gao2022scaling,karwowski2023goodhart,wang2026reproducinganalyzingdetectingreward,mahmoud2026rewardhackingrubric}.
Related work studies rubric reward design and mitigation \citep{gunjal2025rubrics,yang2026rubricdropoutsimpleway}, practical judge failures such as position bias \citep{zheng2023judging,shi2025judgingjudges}, and correlated errors in heterogeneous panels \citep{verga2024replacing,kohli2026judgeseffectivevotescorrelated}. We frame this as a construct-validity question: \emph{what changes can the evaluator's output channel identify at all, and which remain indistinguishable from genuine improvement?}

We study the strongest verdict-only evidence: the \emph{joint verdict
distribution} over all criterion grades, including margins and cross-criterion
agreement. Policies with the same joint verdict distribution are
indistinguishable without held-out gold or other external evidence.

\begin{figure}[t]
    \centering
    \includegraphics[width=0.98\linewidth]{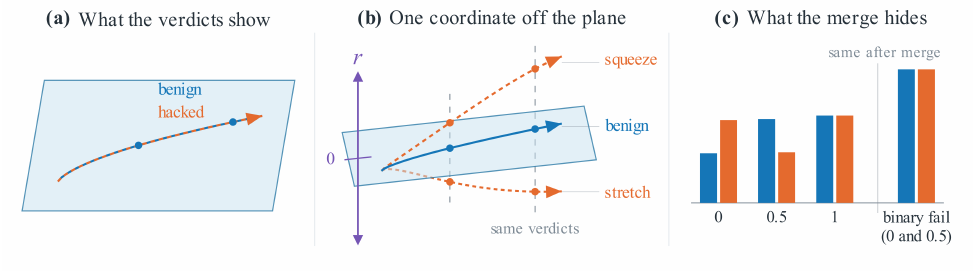}
    \caption{\textbf{Why a third grade helps.}
    \textbf{(a)} The space of binary verdict evidence, seen face-on. A benign run
    and a hacked run (one that only rescales scores around the pass cutoff) trace
    one path: every verdict agrees at every step.
    \textbf{(b)} The same runs against one statistic the panel never reports,
    \(r\). The benign run stays at \(r=0\); the hacked runs ride above and below,
    one squeezing scores toward the cutoff and one stretching them away, at
    identical verdicts throughout. Any statistic that moves with the rescaling
    would do, and the middle grade is one of them.
    \textbf{(c)} With three grades the two runs differ, but merging 0 and 0.5
    into a single fail recovers identical binary evidence. Under Gaussian score
    laws, holding both thresholds fixed forces the rescaling to be the identity;
    within-grade changes can still hide.}
    \label{fig:grading-idea}
\end{figure}

\begin{samepage}
Figure~\ref{fig:grading-idea} illustrates why retaining a middle grade matters. Each criterion is scored as 0, 0.5, or 1; binarization treats both 0 and 0.5 as failure and only 1 as passing. Rescaling scores around the pass
cutpoint (a \emph{stretch}), or moving an output from \(0\) to \(0.5\), can therefore change graded scores without changing any binary
verdict, even jointly across criteria. Retaining \(0.5\) exposes this motion, and
under Gaussian score laws two fixed thresholds eliminate the stretch
(Theorem~\ref{thm:fiber}); three grades are a floor, not a target. The loading-aligned mean shift matches a reference score law, so it stays invisible even with continuous scores. Finer finite scales can expose particular within-grade changes, but every finite scale retains some. External evidence remains necessary to distinguish reference-like motion from genuine improvement.

\paragraph{Contributions.}
\begin{list}{(\arabic{contrib})}{\usecounter{contrib}
  \setlength{\leftmargin}{1.9em}
  \setlength{\labelwidth}{1.4em}\setlength{\labelsep}{0.5em}
  \setlength{\itemsep}{0pt}\setlength{\parskip}{0pt}
  \setlength{\parsep}{0pt}\setlength{\topsep}{0pt}}
  \item We identify latent score changes that leave the full joint binary
  verdict distribution unchanged.
  \item We prove a minimal Gaussian-model fix: retain one middle grade and its
  second threshold.
  \item We verify the fix with MATH and SciBench ablations and characterize
  residual changes requiring external validation. All proofs appear in
  Appendix~\ref{app:proofs}.
\end{list}
\end{samepage}

% =====================================================================
\section{Observation model and identification question}
\label{sec:model}

We call each \((\text{judge},\text{criterion})\) output channel a
\emph{panel member}. The theory permits members from one or many judges, while the
experiments instantiate multiple criteria from one fixed judge for consistency. For output
\(i\) and member \(j\), let the latent judge score be
\[
 S_{ij}=\lambda_j\theta_i+\varepsilon_{ij},\qquad
 \theta_i=m+\sqrt c\,Z_i,
\]
where \(\lambda_j\) is the member loading and
\(Z_i,\varepsilon_{i1},\ldots,\varepsilon_{iJ}\) are independent centered
Gaussians with variances \(1,\psi_1,\ldots,\psi_J\). Hence
\(\Sigma(c)=\operatorname{diag}(\psi)+c\lambda\lambda^\top\). At \(c=1\),
calibration sets \(\psi_j=1-\lambda_j^2\), standardizing scores only there. A binary member returns
\(Y_{ij}=\ind\{S_{ij}>\tau_j\}\). A graded member instead fixes
\(-\infty=\tau_{j,0}<\cdots<\tau_{j,K_j}=+\infty\) and reports the interval
containing \(S_{ij}\). These are measurement assumptions, not claims that the shared factor is gold
quality or the optimum. Section~\ref{sec:exp} also applies the
threshold-preserving transformation directly to realized scores.

We model ordinary (non-hacked) improvement as changes in the shared factor's mean \(m\) and
variance \(c\). Writing
\(\Refset=\{(m,c):m\in\mathbb R,c>0\}\), we call the corresponding
two-parameter surface of joint verdict distributions the \emph{reference orbit}.
This is deliberately narrow: genuine gains may be
criterion-specific and fall outside it.
Reference-arm gains on MATH and SciBench are better explained by this family
than by no motion, although the fit is approximate and fails on HealthBench.

\begin{definition}[Reference compatibility]
\label{def:detect}
A policy change is \emph{detectable from joint verdicts} when its joint verdict
distribution does not equal any distribution in the reference family. Under
the joint-Gaussian model, all cumulative margins and latent-score dependence
determine the joint distribution.
Outside the joint-Gaussian model, summaries that differ from every reference
law still establish departure, but matching summaries need not establish
equality of the joint distributions. A change in a grade margin or pairwise
table is observable; it establishes departure only if no reference law explains it.
\end{definition}

The policy may move latent judge scores but cannot change the declared panel
parameters, cutpoints, or residual process. This fixed-instrument assumption is
what makes the blind set auditable before deployment.

% =====================================================================
\section{Main result: what grading fixes}
\label{sec:blind}

\subsection{One threshold leaves the score scale unidentified}

\begin{observation}[Threshold-preserving transformations are invisible]
\label{obs:threshold-preserving}
Let each criterion assign grades by comparing a latent score with fixed
thresholds. Any strictly increasing coordinatewise transformation that keeps those thresholds fixed leaves every grade (and thus the joint verdict distribution) unchanged for every sampled output under a shared coupling. This means binary grading pins the transformation at
one point while three grades pin it at two. This family is sufficient, not
exhaustive; invisible within-cell rearrangements need not be monotone.
\end{observation}

This observation is distribution-free. Under Gaussian pre- and post-change
score laws, equivalence is affine, so two fixed thresholds force the identity
(Theorem~\ref{thm:fiber}); the one-factor model is not required.

\begin{theorem}[Gaussian equivalence of joint verdict distributions]
\label{thm:fiber}
Let \(S\sim N_J(\mu,\Sigma)\) and \(S^\dagger\sim
N_J(\mu^\dagger,\Sigma^\dagger)\) have positive definite covariances, the same
fixed finite cutpoints, and fixed category labels.
\textbf{(i)} If every member is binary with threshold vector \(\tau\), the
joint verdict distributions agree if and only if there exists a positive diagonal
\(D\) with
\[
\mu^\dagger-\tau=D(\mu-\tau),\qquad
\Sigma^\dagger=D\Sigma D.
\]
Moreover, \(D\) is unique.

\textbf{(ii)} With graded members, the same relations hold about each member's
first finite cutpoint, but necessarily \(D_{jj}=1\) whenever \(K_j\ge3\).
Conversely, these restricted relations imply equality of the joint verdict
distributions. \textnormal{(\hyperref[app:proof-fiber]{Proof: App.~\ref*{app:proof-fiber}.})}
\end{theorem}

\paragraph{What this means.} Each binary coordinate permits a positive affine
stretch around its pass threshold, a classical ordinal-probit indeterminacy
\citep{muthen1984general,olsson1979maximum,chib1998analysis}; when verdicts
supply the optimization reward, it becomes an exactly characterizable blind
set. Even the full joint verdict distribution cannot reveal it. A second fixed threshold forces
the stretch to one. The one-factor model enters only when we characterize which
reference-like changes remain invisible.

\begin{claimbox}
\textbf{Design result: preserve grades rather than binarizing them.} Under
joint-Gaussian scores, a binary criterion adds information and one affine stretch
ambiguity; a second finite threshold removes it. Thus adding binary criteria
buys information and blindness together, and identifiability
is bounded by the coarsest member. Without Gaussianity, extra grades constrain
more points but cannot reveal arbitrary within-grade changes.
\end{claimbox}

\subsection{The blind direction that remains}

\begin{corollary}[Maximality of loading-aligned mean attacks]
\label{cor:blind}
Fix \((m,c)\in\Refset\) and consider a deterministic mean shift
\(S^{\mathrm{att}}\sim N_J(m\lambda+a,\Sigma(c))\). If some pair
\(j\ne k\) has \(\lambda_j\lambda_k\ne0\) and \(\psi_j+\psi_k>0\), the
attacked joint binary or fixed-cutpoint ordinal verdict distribution equals a reference
distribution if
and only if \(a=\delta\lambda\) and \((m',c')=(m+\delta,c)\) for some
\(\delta\in\mathbb R\). \textnormal{(\hyperref[app:proof-blind]{Proof: App.~\ref*{app:proof-blind}.})}
\end{corollary}

\paragraph{What this means.} After the affine stretch ambiguity is removed, a
fixed-covariance mean attack is invisible exactly when it moves every criterion
in proportion to that criterion's sensitivity to the shared factor. To the
panel, that direction looks exactly like ``the model got better.'' It has the
same signature as a sycophancy-shaped stylistic lift that reads as competence
across every criterion. Grading therefore changes the blind set from many
criterion-specific distortions to one interpretable reference-like direction
in this attack class.

\subsection{What grading cannot fix}

The remaining limitation does not rely on Gaussianity or on the number of
grades.

\begin{proposition}[Pathwise imitation of reference motion]
\label{prop:process}
Fix any jointly distributed latent array
\(\{(Z_i,\varepsilon_{i1},\ldots,\varepsilon_{iJ})\}_{i\in\mathcal I}\). No
Gaussianity, independence, or identical sampling is required. For any target
\((m',c')\), the intervention
\[
\Delta_{ij}=\lambda_j\!\left[(m'-m)+(\sqrt{c'}-\sqrt c)Z_i\right]
\]
satisfies \(S_{ij}(m,c)+\Delta_{ij}=S_{ij}(m',c')\) pathwise. Under any fixed
memberwise quantizers, every fixed-horizon, sequential, or verdict-adaptive
rule therefore receives the same evidence from the intervention and the
matched reference process. Adaptive prompt selection additionally requires the
identity conditional on every selectable prompt and history.
\textnormal{(\hyperref[app:proof-process]{Proof: App.~\ref*{app:proof-process}.})}
\end{proposition}

\paragraph{What this means.} This identity marks
the irreducible boundary after the avoidable binary stretch ambiguity is
removed. If
external evidence labels the matched intervention as reward hacking, a
verdict-only level-\(\alpha\) test has power at most \(\alpha\) for every
sample size. When \(c'=c\), the shift is simply \((m'-m)\lambda\). Changing
scale additionally requires the realized shared factor or an exact proxy. This is an existence result; whether ordinary training finds such a policy is open.

% =====================================================================
\section{Off the blind set: three detection regimes}
\label{sec:rates}

Figure~\ref{fig:courtroom-taxonomy}c--d illustrates margin and dependence changes.
Locally, departure from the reference family has the following three regimes.

\begin{figure*}[t]
    \centering
    \includegraphics[width=0.98\textwidth]{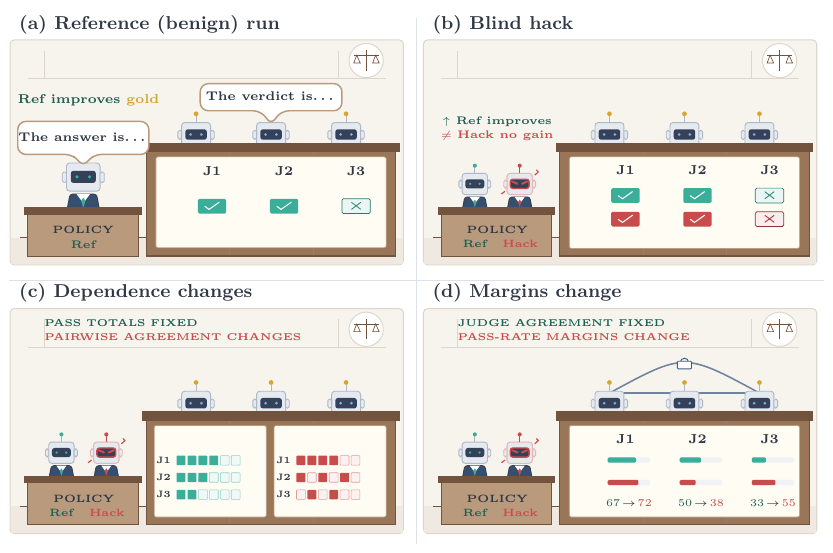}
    \caption{\textbf{Taxonomy of verdict-only evidence.} Schematic examples.
    \textbf{(a)} Reference motion improves verdicts and gold.
    \textbf{(b)} A blind change preserves the joint verdict law while gold worsens.
    \textbf{(c)} Green and red grids show reference and changed pairwise agreement
    at fixed criterion pass totals. \textbf{(d)} Margins change.
    Panels (c--d) illustrate departures unexplained by reference motion.}
    \label{fig:courtroom-taxonomy}
\end{figure*}

\begin{theorem}[Local information geometry]
\label{thm:local}
Minimum KL separation from the reference family is quadratic for first-order
motion away from it, quartic for first-order tangent motion with nonzero normal
relative acceleration, and zero for paths inside the family. Under the standard
i.i.d.\ heuristic \(nD_{\mathrm{KL}}=O(1)\), these correspond, respectively,
to \(n^{-1/2}\)-scale detection, \(n^{-1/4}\)-scale detection, and no power beyond test size.
\textnormal{(\hyperref[app:proof-local]{Full statement and proof:
App.~\ref*{app:proof-local}.})}
\end{theorem}

% =====================================================================
\section{Evidence: the fix, its scope, and its boundary}
\label{sec:exp}

We train \texttt{Qwen/Qwen3-4B-Instruct-2507} with GRPO \citep{shao2024deepseekmath} on MATH-500
\citep{hendrycks2021math,lightman2023verify}, SciBench \citep{wang2023scibench}, and HealthBench
\citep{arora2025healthbench}, scored by a fixed \texttt{Qwen/Qwen3.5-27B} judge on seven
criteria (four for HealthBench). Gold is withheld from training and panel
analysis. We test the first-order and indistinguishable regimes, not the quartic
regime. Appendix~\ref{app:protocol} reports the protocol;
Appendix~\ref{app:reproducibility} gives training and reproducibility details.

\paragraph{Model-implied stretch stress.}
We freeze the seven-criterion geometry.
Rescaling one fitted criterion at a time about its binary cutpoint is a theorem
sanity check: binary complete-law total variation (TV) is zero for all
\(14\) criterion--domain pairs at each stress level. With three grades and a \(1.5\times\) stress, TV is
\(0.0234\)--\(0.1146\) on MATH and \(0.0248\)--\(0.1117\) on SciBench.
At \(n=1{,}024\), estimated power for a level-\(0.05\) likelihood-ratio test
given both population laws was at least \(96.5\%\) on MATH and \(97.7\%\) on
SciBench across all 14 experiments (4,000 repetitions). These fitted-law
stresses are summarized in Figure~\ref{fig:grading}.

\paragraph{Realized-score stress.}
Applying the same threshold-centered transformation directly to realized
scores leaves every binary verdict vector unchanged pathwise, while empirical
three-grade TV is \(0.070\)--\(0.276\) across transforms, so the effect is not
merely a fitted-law artifact.

\begin{figure}[t]
    \centering
    \includegraphics[width=0.88\linewidth]{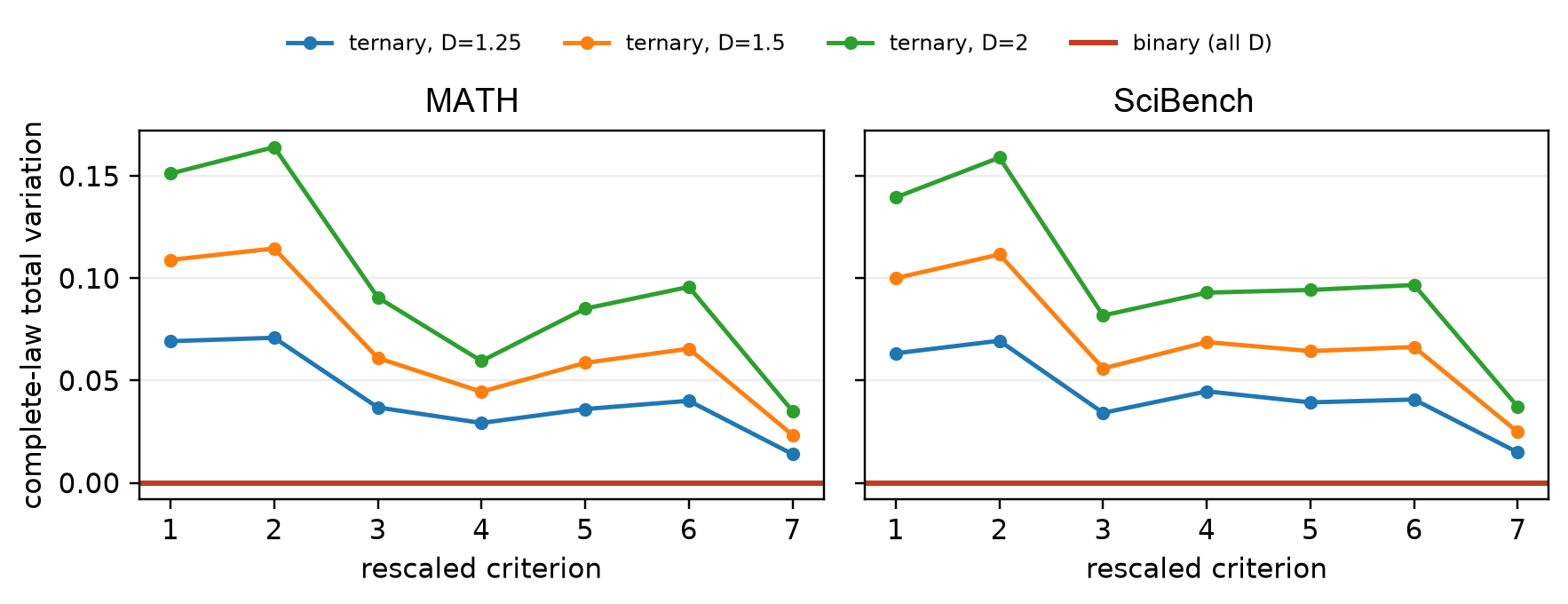}
    \caption{\textbf{Grading resolution removes the fitted binary stretch ambiguity.}
    Each point rescales one criterion. Binary complete-law TV is zero
    throughout; a third grade reveals every direction.}
    \label{fig:grading}
\end{figure}

\paragraph{A constructed residual blind direction.} On frozen MATH
geometry, a constructed genuine-gain law raises the shared factor while an
exploit law lowers it and applies the matched loading-aligned translation.
Mapping the shared factor through the study's gold-accuracy link predicts that
the two constructed policy laws differ by 18.25 percentage points
in implied gold accuracy, yet their panel rewards match and their complete
\(2{,}187\)-cell joint grade tables are numerically identical (population TV
\(1.6\times10^{-16}\)). In a separate construction perturbed away from that
exact equality, TV rises to \(0.0328\) and estimated known-law oracle power is
\(0.9387\) at \(n=1{,}024\). Grading removes the affine stretch ambiguity, but
not this reference-like construction.

\paragraph{The orbit is an approximation.} On
post-calibration reference windows, orbit prediction improves margins and
dependence over a no-motion baseline on MATH and SciBench, but residuals remain
above sampling floors (Table~\ref{tab:orbit}). HealthBench does not beat the no-motion null, perhaps because heterogeneous clinical tasks induce criterion-specific rather than shared-factor motion; we therefore do not interpret its orbit coordinates.

\begin{table}[t]
\centering
\footnotesize
\setlength{\tabcolsep}{4pt}
\caption{\textbf{Reference-arm orbit stress test.} Raw orbit/no-motion errors
and sampling floors (matched-sample noise under a held-fixed fitted law) on
post-calibration windows. Scale gap is
\(|\hat c_{\mathrm{marg}}-\hat c_{\mathrm{dep}}|\); threshold entries report
max raw/residual drift and RMS floor; explained fractions aggregate over
thresholds. HealthBench beats the no-motion null on neither margins nor dependence; we read this as the scope condition failing and do not interpret its orbit coordinates.}
\label{tab:orbit}
\begin{tabular}{@{}lcccc@{}}
\toprule
Domain & \begin{tabular}[c]{@{}c@{}}Margin L1\\orbit/null/floor\end{tabular}
       & \begin{tabular}[c]{@{}c@{}}Dependence RMSE\\orbit/null/floor\end{tabular}
       & \begin{tabular}[c]{@{}c@{}}Cross-channel\\scale gap\end{tabular}
       & \begin{tabular}[c]{@{}c@{}}Threshold drift\\raw/resid./floor; expl.\end{tabular} \\
\midrule
MATH        & \(0.119/0.322/0.013\) & \(0.089/0.206/0.013\) & \(0.304\) & \(0.956/0.236/0.029\);~67\% \\
SciBench    & \(0.076/0.314/0.013\) & \(0.043/0.070/0.019\) & \(0.046\) & \(0.767/0.203/0.037\);~72\% \\
HealthBench & \(0.142/0.106/0.010\) & \(0.149/0.080/0.024\) & \(0.262\) & \(0.260/0.157/0.017\);~35\% \\
\bottomrule
\end{tabular}
\end{table}

\paragraph{Cutpoint stability is monitored.}
Fixed cutpoints are required for \(D=I\). Residual drift exceeds its sampling
floor on MATH (\(0.236\) vs.\ \(0.029\)) and SciBench (\(0.203\) vs.\
\(0.037\)); monitor drift and middle-grade occupancy and recalibrate as needed.

\paragraph{Departure was reachable; natural exact blindness was not.}
Ten signed presentation-over-soundness trajectories depart, but four
candidate-pool searches find no natural exact blind attack. The blind set is
constructible but was not observed in ordinary training.

\phantomsection
\label{sec:claims}
\begin{claimbox}
\textbf{We claim.} (1) Binary panels have a coordinatewise affine stretch
ambiguity under the Gaussian model; fixed ternary grading removes it. (2) The
effect appears for every criterion in two frozen real geometries. (3) The
remaining loading-aligned family is exact in the declared attack class. (4)
Off-orbit information scales with Fisher-normal displacement.

\smallskip
\textbf{We do not claim.} That grading detects all reward hacking; a validated
detector; or that orbit membership implies benignity. Stresses are constructed;
the judge and residual correlation are unvalidated; replay is same-run and
inexact; and HealthBench fails the reference-motion check.
\end{claimbox}

% =====================================================================
\section{What a practitioner should do}
\label{sec:implications}

\paragraph{Preserve resolution.} When grades are available, do not collapse
them to pass/fail. Under the joint-Gaussian model, three fixed grades and
two fixed thresholds are the fewest that suffice; finer scales, up to a
continuous score, are no worse but remove no further stretch and add cutpoints
to hold fixed (App.~\ref{app:protocol}). This recommendation concerns graded
judges. Every finite scale leaves some within-grade changes invisible; even
continuous scores cannot distinguish loading-aligned shifts from reference motion.

\paragraph{Monitor reduced verdict statistics.} Under the Gaussian model for an informative panel,
grade margins and pairwise tables identify \((m,c)\) without the sparse
\(3^J\)-cell joint grade table (\(2{,}187\) cells here). Freeze the instrument, fit by
pairwise composite likelihood, and bootstrap the residual to flag departures.

\paragraph{Pair verdicts with external evidence.} On-orbit silence is
uninformative, and off-orbit movement may be genuine improvement. Treat
departure as a trigger for held-out gold, an independent judge, or response
inspection and not as a diagnosis.

These results address JUDGe facets 01, 02, and 04 and, for multi-judge panels,
07 \citep{judge2026taxonomy}; Appendix~\ref{app:protocol} gives the full taxonomy and disclosure mapping.

\begin{claimbox}
\footnotesize
\textbf{Identification limits for our seven-criterion panel.}
\textit{Training data provenance:} Fixed \texttt{Qwen/Qwen3.5-27B} judge;
training-data sources are not documented here, and the one-factor residual
model is unvalidated. \textit{Deployment context:} Judge scores supply the
GRPO reward; task-specific gold is withheld from training and verdict-only
analysis. \textit{Known failure modes:} Criterion-specific rescalings are
removed; loading-aligned mean shifts remain invisible. \textit{Human
validation:} None; benchmark gold is not a human--judge agreement study.
\end{claimbox}

\clearpage
\bibliographystyle{plainnat}
\bibliography{references}

\clearpage
\appendix
\section{Proofs}
\label{app:proofs}

This appendix proves every formal result stated in the main text. We first
record the one property of the reference family used in two of the proofs.

\subsection{Reference-orbit identification}

Write \(\Psi=\operatorname{diag}(\psi)\),
\(\Sigma(c)=\Psi+c\lambda\lambda^\top\), and
\[
 s_j(c)^2=\operatorname{Var}(S_j)=\psi_j+c\lambda_j^2.
\]
For a finite cutpoint \(\tau_{j\ell}\), let
\(p_{j\ell}=\Pr(S_j>\tau_{j\ell})\), and let
\(R_{jk}(c)=\operatorname{Corr}(S_j,S_k)\).

\begin{lemma}[Identification of the reference orbit]
\label{lem:orbit-identification}
Suppose some pair \(j\ne k\) satisfies
\(\lambda_j\lambda_k\ne0\) and \(\psi_j+\psi_k>0\). Then the joint
Gaussian verdict distribution identifies \((m,c)\). The map
\((m,c)\mapsto q_{m,c}\) of reference parameters to verdict-cell
probabilities is smooth, one-to-one, and has rank two at every interior law.
\end{lemma}

\begin{proof}
For the informative pair,
\[
 R_{jk}(c)=
 \frac{c\lambda_j\lambda_k}{s_j(c)s_k(c)}.
\]
The factor multiplying \(\lambda_j\lambda_k\) is one-to-one in \(c\), since
\[
 \frac{d}{dc}\log\frac{c}{s_j(c)s_k(c)}
 =\frac{1}{2c}\left\{
   \frac{\psi_j}{s_j(c)^2}+\frac{\psi_k}{s_k(c)^2}
 \right\}>0.
\]
Thus \(R_{jk}(c)\) identifies \(c\). At any finite cutpoint,
\[
 p_{j\ell}
 =\Phi\!\left(\frac{m\lambda_j-\tau_{j\ell}}{s_j(c)}\right),
\qquad
 m=\frac{s_j(c)\Phi^{-1}(p_{j\ell})+\tau_{j\ell}}{\lambda_j},
\]
so a nonzero loading then identifies \(m\).

It remains only to connect these quantities to the joint verdict distribution.
Its univariate
margins give every \(p_{j\ell}\). For the informative pair, a bivariate
cumulative cell probability has the form \(\Phi_2(a,b;R_{jk})\), with finite
\(a,b\) already known from the margins. Plackett's identity
\citep{plackett1954reduction} gives
\(\partial\Phi_2(a,b;\rho)/\partial\rho=\phi_2(a,b;\rho)>0\), so that cell
probability identifies \(R_{jk}\). This proves injectivity. Smoothness is
immediate. The \(c\)-derivative has a nonzero correlation component, whereas
the \(m\)-derivative has zero correlation component and a nonzero margin
component, so the two derivatives are linearly independent.
\end{proof}

\subsection{Gaussian equivalence of joint verdict distributions}
\label{app:proof-fiber}

\begin{proof}[Proof of Theorem~\ref{thm:fiber}]
We prove the mixed binary--graded statement; the all-binary case is its special
case. Write
\[
 \sigma_j^2=\Sigma_{jj},\qquad
 (\sigma_j^\dagger)^2=\Sigma^\dagger_{jj},
\]
and let \(R,R^\dagger\) be the corresponding correlation matrices. Equality
of the marginal verdict distributions implies, at every finite cutpoint,
\[
 \frac{\tau_{j\ell}-\mu_j}{\sigma_j}
 =
 \frac{\tau_{j\ell}-\mu_j^\dagger}{\sigma_j^\dagger}.
 \tag{A.1}
\]
Indeed, both sides are obtained by applying \(\Phi^{-1}\) to the same
cumulative category probability.

For \(j\ne k\), equality of a bivariate cumulative verdict probability at one
finite cutpoint from each member gives
\[
 \Phi_2(a_{j\ell},a_{kr};R_{jk})
 =\Phi_2(a_{j\ell},a_{kr};R_{jk}^\dagger),
\]
where the standardized cutpoints agree by (A.1). Positive definiteness puts
both correlations in \((-1,1)\), and Plackett's identity makes the bivariate
normal cdf strictly increasing in its correlation. Hence \(R^\dagger=R\).

Set \(d_j=\sigma_j^\dagger/\sigma_j>0\). If member \(j\) is binary, (A.1)
at its only finite threshold \(\bar\tau_j\) yields
\[
 \mu_j^\dagger-\bar\tau_j
 =d_j(\mu_j-\bar\tau_j).
 \tag{A.2}
\]
If \(K_j\ge3\), subtracting (A.1) at two distinct finite cutpoints gives
\(\sigma_j^\dagger=\sigma_j\), so \(d_j=1\), and then (A.1) gives
\(\mu_j^\dagger=\mu_j\). Finally,
\[
 \Sigma_{jk}^\dagger
 =\sigma_j^\dagger\sigma_k^\dagger R_{jk}^\dagger
 =d_jd_k\Sigma_{jk}.
\]
Thus the stated relations hold with the unique matrix
\(D=\operatorname{diag}(d_1,\ldots,d_J)\).

Conversely, suppose those relations hold, choosing \(\bar\tau_j\) as the sole
threshold for a binary member and the first finite cutpoint for a graded one.
The Gaussian vector
\[
 \widetilde S=\bar\tau+D(S-\bar\tau)
\]
has mean \(\mu^\dagger\) and covariance \(\Sigma^\dagger\), hence the law of
\(S^\dagger\). Positive scaling about a binary threshold preserves its label.
Every graded coordinate has \(d_j=1\) and is unchanged. Therefore the entire
labeled verdict vector is preserved pathwise, which proves equality of the
joint verdict distributions.
\end{proof}

\subsection{The remaining loading-aligned family}
\label{app:proof-blind}

\begin{proof}[Proof of Corollary~\ref{cor:blind}]
Suppose the attacked joint verdict distribution equals a reference
distribution at \((m',c')\).
For the pair in the corollary, the determinant of its covariance block is
\[
 c\lambda_j^2\psi_k+c\lambda_k^2\psi_j+\psi_j\psi_k>0,
\]
so the pair is nonsingular. Equality of its bivariate verdict distribution identifies
its latent correlation by the same Plackett argument as above. The
one-to-one correlation formula in Lemma~\ref{lem:orbit-identification} then
forces \(c'=c\).

The attacked and reference marginal variances are consequently equal for
every member. Equality of a cumulative probability at any common finite
cutpoint identifies the corresponding marginal means, giving
\[
 m\lambda+a=m'\lambda.
\]
Therefore \(a=(m'-m)\lambda\). Conversely, if
\(a=\delta\lambda\), then the attacked latent law is exactly the reference
law at \((m+\delta,c)\). This proves both necessity and sufficiency.
\end{proof}

\subsection{Pathwise imitation}
\label{app:proof-process}

\begin{proof}[Proof of Proposition~\ref{prop:process}]
Direct substitution gives, simultaneously for every \(i,j\),
\begin{align*}
 S_{ij}(m,c)+\Delta_{ij}
 &=\lambda_j(m+\sqrt c\,Z_i)+\varepsilon_{ij}
   +\lambda_j\{(m'-m)+(\sqrt{c'}-\sqrt c)Z_i\}\\
 &=\lambda_j(m'+\sqrt{c'}Z_i)+\varepsilon_{ij}
 =S_{ij}(m',c').
\end{align*}
Fixed quantizers therefore produce identical verdict arrays under this
coupling. Any fixed-horizon or stopping-time rule is a measurable function of
that array and, if randomized, an auxiliary random seed that can be shared
under the coupling. Its output is consequently identical pathwise. The same
induction covers adaptive prompt selection when the identity holds
conditionally for every prompt and history the rule can select.
\end{proof}

\subsection{Local information geometry}
\label{app:proof-local}

\paragraph{Full statement.}
Let \(\mathcal M=\{q_\eta:\eta\in\Refset\}\) be the reference cell
probabilities and let \(\Pi_{\mathcal T}\) project onto the reference family's
tangent space in Fisher geometry. Assume the informative-pair condition of
Lemma~\ref{lem:orbit-identification} and an interior reference law with strictly
positive cell probabilities. For a smooth joint-verdict-distribution path
\(p_t\) through an interior reference distribution, write
\(u=\dot p_0\) and \(u_\perp=(I-\Pi_{\mathcal T})u\). Then
\[
\inf_\eta D_{\mathrm{KL}}(p_t\Vert q_\eta)
=\tfrac{t^2}{2}\|u_\perp\|_{p_0}^2+o(t^2).
\]
If \(u_\perp=0\), the leading term is
\(\tfrac{t^4}{8}\|\kappa\|_{p_0}^2+o(t^4)\), where \(\kappa\) is the
normal relative acceleration after subtracting the reference family's
curvature.

Let \(\mathcal Y\) be the finite verdict alphabet and
\(p_0=q_{\eta_0}\) an interior reference distribution. On the zero-sum tangent space,
use the Fisher inner product
\[
 \langle v,w\rangle_{p_0}
 =\sum_{y\in\mathcal Y}\frac{v(y)w(y)}{p_0(y)}.
\]
Write \(Q=Dq_{\eta_0}\) for the Jacobian of the reference map and
\(B=D^2q_{\eta_0}\) for its second derivative. Then
\(\mathcal T=\operatorname{range}(Q)\). If \(u\in\mathcal T\), the rank-two
property gives a unique \(a\) with \(u=Qa\); with
\(v=\ddot p_0\), the normal relative acceleration in
Theorem~\ref{thm:local} is
\[
 \kappa=(I-\Pi_{\mathcal T})\{v-B[a,a]\}.
 \tag{A.3}
\]

\begin{proof}[Proof of Theorem~\ref{thm:local}]
For positive laws \(p,q\to p_0\) on a finite alphabet, Taylor expansion gives
\[
 D_{\mathrm{KL}}(p\Vert q)
 =\frac12\lVert p-q\rVert_{p_0}^2
  +o(\lVert p-q\rVert_{p_0}^2).
 \tag{A.4}
\]
The linear term vanishes because both laws sum to one.

First consider the local infimum. For a reference curve
\(\eta(t)=\eta_0+tb+o(t)\),
\[
 p_t-q_{\eta(t)}=t(u-Qb)+o(t).
\]
Minimizing the leading Fisher norm over \(b\) projects \(u\) onto
\(\mathcal T\). Equation (A.4) therefore gives
\[
 \inf_{\eta\ \mathrm{near}\ \eta_0}
 D_{\mathrm{KL}}(p_t\Vert q_\eta)
 =\frac{t^2}{2}\lVert u_\perp\rVert_{p_0}^2+o(t^2).
 \tag{A.5}
\]
Formally, \(Q\) has full column rank by
Lemma~\ref{lem:orbit-identification}, so the Hessian of
\(\eta\mapsto D_{\mathrm{KL}}(p_0\Vert q_\eta)\) is the positive-definite
matrix \(Q^\top\operatorname{diag}(p_0^{-1})Q\). The implicit-function
theorem thus supplies the unique smooth local minimizer used in (A.5).

Now suppose \(u=Qa\). A second-order reference curve has the form
\(\eta(t)=\eta_0+ta+t^2b/2+o(t^2)\), and hence
\[
 p_t-q_{\eta(t)}
 =\frac{t^2}{2}\{v-B[a,a]-Qb\}+o(t^2).
\]
Minimizing over \(b\) removes the tangent component and leaves \(\kappa\)
from (A.3). Applying (A.4) again yields
\[
 \inf_{\eta\ \mathrm{near}\ \eta_0}
 D_{\mathrm{KL}}(p_t\Vert q_\eta)
 =\frac{t^4}{8}\lVert\kappa\rVert_{p_0}^2+o(t^4).
 \tag{A.6}
\]

It remains to show that the infima in (A.5)--(A.6) are global. Fix a
neighborhood \(U\) of \(\eta_0\). If a sequence
\(\eta_n=(m_n,c_n)\notin U\) satisfied
\(\operatorname{TV}(q_{\eta_n},p_0)\to0\), all of its univariate and
bivariate verdict margins would converge to those of \(p_0\). At the
informative pair, continuous inversion of the bivariate Gaussian cdf would
give \(R_{jk}(c_n)\to R_{jk}(c_0)\). The strict monotonicity in
Lemma~\ref{lem:orbit-identification} implies \(c_n\to c_0\), and any finite
cutpoint with nonzero loading then gives
\[
 m_n=
 \frac{s_j(c_n)\Phi^{-1}(p_{j\ell,n})+\tau_{j\ell}}{\lambda_j}
 \longrightarrow m_0,
\]
contradicting \(\eta_n\notin U\). Hence
\(\delta_U:=\inf_{\eta\notin U}\operatorname{TV}(p_0,q_\eta)>0\).
For small enough \(t\), the triangle inequality and Pinsker's inequality give
\[
 \inf_{\eta\notin U}D_{\mathrm{KL}}(p_t\Vert q_\eta)
 \ge \delta_U^2/2,
\]
whereas the local comparison tends to zero. Thus no parameter outside \(U\)
can attain or approach the global infimum, so the global and local infima
agree. This completes the proof.
\end{proof}

\section{Experimental protocol and evidence scope}
\label{app:protocol}

\paragraph{Policies, tasks, and judge.}
\texttt{Qwen/Qwen3-4B-Instruct-2507} policies are trained with GRPO \citep{shao2024deepseekmath} on MATH-500, SciBench, and
HealthBench. One fixed \texttt{Qwen/Qwen3.5-27B} judge scores each response on
seven criteria for MATH and SciBench and four for HealthBench. MATH and
SciBench use the grading and binarization rule reported in
Section~\ref{sec:intro}. Task-specific external gold is withheld from both
training and verdict-only analysis. 

\paragraph{Frozen measurement geometry.}
For the grading-resolution ablation, thresholds, loadings, and residual
variances are estimated from \(8{,}192\) early rollouts per domain and then
held fixed. Finite thresholds are recovered from cumulative graded margins by
probit inversion; pairwise latent correlations are obtained from the
polychoric Gaussian model; and the one-factor geometry is fit to the resulting
latent correlation structure. Middle-grade occupancy pools all seven criteria:
\(9{,}526/57{,}344=16.6\%\) on MATH and
\(9{,}721/57{,}344=17.0\%\) on SciBench; each criterion contributes
\(8{,}192\) ratings. Post-calibration replay fits the earliest at
most eight steps of each run, excludes calibration-overlapping evaluation
windows, and then freezes the geometry. The replay covers five
reference-designated and five stress-test runs in each of MATH and SciBench,
plus five reference-designated and three adversarial stress runs in HealthBench.
Reference designation is a protocol label and does not by itself establish
semantic benignity.

\paragraph{Why three grades, and why not more.}
Three is the fewest that works, and under the joint-Gaussian model it is also
enough. In the proof of Theorem~\ref{thm:fiber}, two distinct finite cutpoints
force \(\sigma_j^\dagger=\sigma_j\), hence \(d_j=1\) and
\(\mu_j^\dagger=\mu_j\): the entire marginal score law of that member is pinned
once \(K_j\ge3\), so a fourth or fifth grade removes no further stretch.
Identification is memberwise, so the operative rule is that no member may be
binary, not that every member carry exactly three grades; a single binary
member reintroduces one stretch on its own coordinate. A fully continuous score
would not reach what remains. Both residual results are equalities of the score
law itself rather than quantization artifacts: the loading-aligned attack of
Corollary~\ref{cor:blind} satisfies
\(N_J(m\lambda+\delta\lambda,\Sigma(c))=N_J((m+\delta)\lambda,\Sigma(c))\),
which is exactly the reference law at \((m+\delta,c)\), and the intervention of
Proposition~\ref{prop:process} matches the reference process pathwise. Neither
is visible to any function of the scores, at any resolution. Finer scales also
carry costs our own measurements show are not hypothetical: \(D=I\) requires
every cutpoint to stay fixed, and residual cutpoint drift already exceeds its
sampling floor (\(0.236\) vs.\ \(0.029\) on MATH, \(0.203\) vs.\ \(0.037\) on
SciBench); identifying a cutpoint requires mass on both of its sides, and the
middle grade already falls to \(7.0\%\) on the thinnest criterion; and the
joint grade table grows as \(K^J\), from \(2{,}187\) cells at \(K=3\) and
\(J=7\) to \(78{,}125\) at \(K=5\), which is why
Section~\ref{sec:implications} recommends margins and pairwise tables instead.
We did not test judges that emit continuous scores, so this comparison is a
statement about the model and not an empirical one.

\paragraph{Distances, sampling floor, and oracle power.}
For categorical distributions \(p\) and \(q\), total variation is
\[
  \operatorname{TV}(p,q)=\tfrac12\sum_y |p(y)-q(y)|.
\]
Known-law power uses the likelihood-ratio test between two fixed population
distributions at level \(0.05\), estimated from 4,000 Monte Carlo repetitions
with \(n=1{,}024\) samples per repetition. It is an upper benchmark because an
operational monitor must estimate its reference law. Sampling floors are
generated from matched-size parametric samples of a held-fixed fitted law,
without rerunning the full fit; they are descriptive finite-sample noise
levels, not confidence bounds.

\paragraph{Constructed and observed searches.}
The model-implied stress rescales one fitted criterion at a time about its
binary threshold, leaving all other coordinates fixed. The direct
realized-score stress applies the corresponding transformation to observed
scores. The residual-direction construction compares a shared-factor gain
with a matched loading-aligned translation and then perturbs transversely away
from that exact equality. Candidate-pool searches use four-to-eight-step
windows (approximately \(4{,}096\)--\(8{,}192\) rollouts). The reweighting harm
gate requires at least a five-point gold-accuracy loss, and the component
tolerance bounds every transverse moment residual by \(0.025\). For the
weak-verifier screen, breadth requires high-verifier incorrect selections on
at least 10 prompts and in all five policy seeds; validity requires at least
99\% parse-valid and at most 1\% unscorable responses; and truncation permits
at most 1\% hard-truncated responses. Gold-blind reweighting gains \(3.07\)
percentage points of signed reward and loses \(4.40\) points of gold accuracy,
missing both gates. Presentation-only selection raises both
presentation (\(6.22\) points) and gold accuracy (\(5.66\) points). The
matched-control construction stops before reading gold because it cannot
optimize the required control under the frozen geometry.
Weak-verifier best-of-\(N\) gains \(0.1556\) proxy score, but its accuracy is
\(0.52\) points above matched random (95\% seed-first, prompt-second
percentile-bootstrap CI \([-3.13,4.69]\)) and its breadth, validity, and
truncation gates fail. These pool-specific failures do not establish that
trained policies can never enter the blind set.

\paragraph{JUDGe taxonomy and disclosure mapping.}
The strongest fit is facet 04: loading-aligned shifts formalize a blind
judge--training feedback direction. Fixed-cutpoint stability connects to facet
02; when panel members are distinct judges, the dependence analysis connects to
facet 07. The result is also a dual of facet 01: distinct policy-induced score
laws can produce identical verdicts. The main-text box instantiates all four
Judge Deployment Disclosure dimensions \citep{judge2026taxonomy}.

\paragraph{Additional operational disclosures.}
\textbf{Output resolution.} Three fixed grades; two cutpoints held fixed; no
pass/fail collapse. \textbf{Oracle benchmark.} At \(n=1{,}024\), three grades
reveal all 14 criterion--domain pairs at the \(1.5\times\) stress with
known-law power at least \(96.5\%\);
this is not operational power. \textbf{External trigger.} For a
residual-direction gain or failed fit, require held-out gold, an independent
judge, or response inspection.

% =====================================================================
\section{Reproducibility details for the confirmatory runs}
\label{app:reproducibility}

\paragraph{Scope, pairing, and intervention.}
This section documents the MATH and SciBench confirmatory trajectories used in
the replay analyses. Each domain crosses five registered seeds
\(\{1,2,3,4,5\}\) with benign and hacked arms, for ten runs per domain and
twenty runs total. Within a seed, the two arms use the same
\texttt{data.seed} and shuffled data order. Each run starts afresh from
\texttt{Qwen/Qwen3-4B-Instruct-2507}; checkpoint resumption is disabled. The
only arm-level intervention is the criterion weighting in the training reward.
For MATH, the hacked weights are
\(w_{\rm presentation}=1.0\), \(w_{\rm thoroughness}=0.0\), and
\(w_{\rm soundness}=-0.6\). For SciBench, they are
\(w_{\rm presentation}=1.0\) and \(w_{\rm soundness}=-0.6\). In both benign
arms the weight variable is unset, which invokes the uniform mean of all seven
criteria. All other data, optimization, sampling, and serving settings are
shared across arms.

Table~\ref{tab:repro-matrix} reports untruncated MATH windows (436 steps; 135
remain after the coherence-collapse gate) and SciBench windows ending near
collapse (180 steps; 149 remain).

\begin{center}
  \centering
  \small
  \captionof{table}{\textbf{Confirmatory run matrix.} ``Post-reference'' excludes the fixed
  reference epoch at steps 0--8. Configured horizons and retained artifact
  lengths are separated because two MATH benign runs retain legacy 150-step
  trajectories.}
  \label{tab:repro-matrix}
  \begin{tabularx}{\textwidth}{@{}>{\bfseries\RaggedRight\arraybackslash}p{0.18\textwidth}YY@{}}
    \toprule
      & MATH & SciBench \\
    \midrule
    Seeds and arms
      & Seeds 1--5; paired benign and hacked runs
      & Seeds 1--5; paired benign and hacked runs \\
    Benign launch and artifacts
      & Amended cap 60; seeds 2 and 3 retain legacy 150-step artifacts.
        Collapse onset is at steps 37--39. Post-reference windows per seed are
        52, 141, 139, 52, and 52; 135 are gate-eligible after pooling.
      & Nominal launch horizon 150. Collapse occurs near step 41.
        Post-reference windows per seed are 39, 37, 37, 35, and 32; 149 are
        gate-eligible after pooling. \\
    Hacked artifacts
      & Horizon 100; all five seeds reach step 100.
      & Horizon 100; four seeds reach step 100 and seed 3 ends at step 98. \\
    \bottomrule
  \end{tabularx}
\end{center}

\paragraph{GRPO and optimizer.}
We use GRPO (\texttt{algorithm.adv\_estimator=grpo}). Each optimizer step
samples 64 prompts with 16 policy rollouts per prompt, yielding 1,024 responses
per step. The PPO mini-batch size is 32 and the per-GPU micro-batch size is 4.
The full-parameter actor uses AdamW with learning rate \(3\times10^{-6}\),
betas \((0.9,0.999)\), weight decay 0.01, a constant schedule, 15 warmup steps,
and gradient-norm clipping at 1.0. Gradient checkpointing is enabled. FSDP
parameter and optimizer offload are enabled for the actor, and parameter
offload is enabled for the reference policy. The actor KL loss is enabled with
coefficient \(10^{-3}\) and the low-variance KL form; KL is not added to the
reward and the entropy coefficient is zero.

\paragraph{Policy and judge sampling.}
Training and validation policy sampling both use temperature 0.7,
\(\texttt{top\_p}=0.8\), \(\texttt{top\_k}=20\), and sampling enabled. The
maximum response length is 2,048 tokens; maximum prompt length is 1,024 tokens
for MATH and 1,536 for SciBench. The policy rollout engine uses a 4,096-token
model context and at most 16,384 batched tokens. MATH applies the same
training-only response-length penalty in both arms: it begins at length 5,500,
is hard at 8,000, and has coefficient 1.0. SciBench has no length penalty.

The fixed \texttt{Qwen/Qwen3.5-27B} judge is served separately with vLLM on two
GPUs using tensor parallelism 2, bfloat16, and a 32,768-token server context.
Judge calls use temperature 1.0, at most 7,700 generated tokens, and thinking
disabled. An incomplete seven-criterion parse is retried up to three times.
The task-specific gold signal is computed independently and is excluded from
both the reward and verdict-only monitoring.

\paragraph{Checkpoint and analysis-window policy.}
There is deliberately no best-checkpoint selection. Checkpoints are written
every 25 optimizer steps, validation runs every 20 steps, and
\texttt{resume\_mode=disable}. A monitor alarm never halts a hacked run; the
configured horizon is the planned stopping rule, and any incidental early
termination is retained and reported. The analysis selects a trajectory window
rather than a checkpoint. Its fixed reference epoch is steps 0--8, and
the clean window ends immediately before coherence collapse, defined for this
analysis as the first post-reference step whose mean log-probability per token
is below the run's reference mean by 0.5. Executed horizons are determined by
the frozen launcher configurations and recorded rollout artifacts.

\paragraph{Compute and runtime accounting.}
One run uses one node with six H100 GPUs: four for policy training and two for
the judge server. The job timeout is 86,400 seconds and at most three runs are
concurrent, for a peak allocation of 18 H100s. The repository artifacts do not
record per-run start and end timestamps, so we do not report measured
H100-hours. For scale, one 24-hour attempt for each of the 20 cells would be a
configured ceiling of \(20\times6\times24=2{,}880\) H100-hours, excluding
pilots and retries; actual consumption is lower for trajectories that end
before the timeout.

\end{document}